\documentclass[letterpaper]{article} 
\usepackage{aaai24}  
\usepackage{times}  
\usepackage{helvet}  
\usepackage{courier}  
\usepackage[hyphens]{url}  
\usepackage{graphicx} 
\usepackage{natbib}  
\usepackage{caption} 
\usepackage[utf8]{inputenc} 
\usepackage[T1]{fontenc}    
\usepackage{booktabs}       
\usepackage{amsfonts}       
\usepackage{nicefrac}       
\usepackage{microtype}      
\usepackage{xcolor}         

\usepackage{graphicx}
\usepackage{subcaption}
\usepackage{float}
\usepackage{enumitem}
\usepackage{multirow}
\usepackage{colortbl}
\usepackage{makecell}
\usepackage[american]{babel}

\usepackage{amsmath}
\usepackage{amssymb}
\usepackage{mathtools}
\usepackage{amsthm}
\usepackage{bm}
\usepackage{thmtools}
\usepackage{thm-restate}

\usepackage[noabbrev,nameinlink,capitalise]{cleveref}

\theoremstyle{plain}
\newtheorem{theorem}{Theorem}[section]

\newtheorem{lemma}[theorem]{Lemma}

\theoremstyle{definition}
\newtheorem{definition}[theorem]{Definition}

\theoremstyle{remark}

\DeclareRobustCommand{\eagleimg}{%
  \begingroup\normalfont \includegraphics[height=1.5\fontcharht\font`\B]{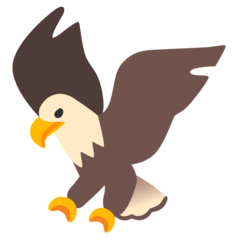}%
  \endgroup
}
\newcommand{\eagle}{\eagleimg}
\newcommand{\predmodel}{M_{\bm{\theta}}}
\newcommand{\feature}{x}
\newcommand{\features}{\bm{x}}
\newcommand{\singlelabel}{y}
\newcommand{\labels}{\bm{\singlelabel}} \newcommand{\pred}{\hat{\singlelabel}} \newcommand{\preds}{\bm{\pred}}
\newcommand{\sample}{\Tilde{\singlelabel}} \newcommand{\samples}{\bm{\sample}} \newcommand{\decision}{\bm{z}}
\newcommand{\optfn}{\decision^*}
\newcommand{\obj}{f}
\newcommand{\dqregret}{DQ_{\text{regret}}}
\newcommand{\pto}{Predict-then-Optimize}

\DeclareMathOperator*{\argmin}{arg\,min} \DeclareMathOperator*{\argmax}{arg\,max}
 \newcommand\numberthis{\addtocounter{equation}{1}\tag{\theequation}}

\definecolor{figurered}{RGB}{197,76,73}
\definecolor{figuregreen}{RGB}{67,140,78}
\definecolor{figureblue}{RGB}{39,97,164}
\definecolor{figureorange}{RGB}{235,111,27}
\definecolor{lodlblue}{RGB}{13,64,147}
\definecolor{ForestGreen}{RGB}{34,139,34}
\allowdisplaybreaks
\title{Leaving the Nest \eagle{}:\\
Going Beyond Local Loss Functions for Predict-Then-Optimize}

\author {
    Sanket Shah\textsuperscript{\rm 1}, Bryan Wilder\textsuperscript{\rm 2}, Andrew Perrault\textsuperscript{\rm 3},
    Milind Tambe\textsuperscript{\rm 1} } \affiliations {
    \textsuperscript{\rm 1}Harvard University\\
    \textsuperscript{\rm 2}Carnegie Mellon University\\
    \textsuperscript{\rm 3}Ohio State University\\
    sanketshah@g.harvard.edu, bwilder@andrew.cmu.edu, perrault.17@osu.edu, milind\_tambe@harvard.edu }

\begin{document}
\maketitle
\begin{abstract}
\pto{} is a framework for using machine learning to perform decision-making under uncertainty. The central research
question it asks is, ``How can we use the structure of a decision-making task to tailor ML models for that specific
task?'' To this end, recent work has proposed learning task-specific loss functions that capture this underlying
structure. However, current approaches make restrictive assumptions about the form of these losses and their impact on
ML model behavior. These assumptions both lead to approaches with high computational cost, and when they are violated in
practice, poor performance. In this paper, we propose solutions to these issues, avoiding the aforementioned assumptions
and utilizing the ML model's features to increase the sample efficiency of learning loss functions. We empirically show
that our method achieves state-of-the-art results in four domains from the literature, often requiring an order of
magnitude fewer samples than comparable methods from past work. Moreover, our approach outperforms the best existing
method by nearly 200\% when the localness assumption is broken.
\end{abstract}

\section{Introduction}
Predict-then-Optimize (PtO)~\citep{donti2017task,elmachtoub2021smart} is a framework for using machine learning (ML) to
perform decision-making under uncertainty. As the name suggests, it proceeds in two steps---first, an ML model is used
to make predictions about the uncertain quantities of interest, then second, these predictions are aggregated and used
to parameterize an optimization problem whose solution provides the decision to be made. Many real-world applications
require both prediction and optimization, and have been cast as PtO problems. For example, recommender systems need to
predict user-item affinity to determine which titles to display~\citep{wilder2019melding}, while portfolio optimization
uses stock price predictions to construct high-performing portfolios~\citep{bengio1997using}. In the context of AI for
Social Good, PtO has been used to plan intervention strategies by predicting how different subgroups will respond to
interventions~\citep{wang2022decision}.

The central research question of PtO is, ``How can we use the structure of an optimization problem to learn predictive
models that perform better \textit{for that specific decision-making task}?". In this paper, we refer to the broad class
of methods used to achieve this goal as \textit{Decision-Focused Learning} (DFL). Recently, multiple papers have
proposed learning task-specific loss functions for DFL~\citep{chung2022decision,lawless2022note,shah2022decision}. The
intuition for these methods can be summarized in terms of the Anna Karenina principle---while perfect predictions lead
to perfect decisions, different kinds of imperfect predictions have different impacts on downstream decision-making.
Such loss functions, then, attempt to use learnable parameters to capture how bad different kinds of prediction errors
are for the decision-making task of interest. For example, a Mean Squared Error (MSE) loss may be augmented with tunable
parameters to assign different weights to different true labels. Then, a model trained on such a loss is less likely to
make the kinds of errors that affect the quality of downstream decisions.

Learning task-specific loss functions poses two major challenges. First, learning the relationship between predictions
and decisions is challenging. To make learning this relationship more tractable, past approaches learn different loss
functions for each instance of the decision-making task, each of which locally approximates the behavior of the
optimization task. However, the inability to leverage training samples across different such instances can make learning
loss functions sample inefficient, especially for approaches that require a large number of samples to learn. This is
especially problematic because creating the dataset for learning these loss functions is the most expensive part of the
overall approach.  In this paper, rather than learning separate loss functions for each decision-making instance, we
learn a mapping from the feature space of the predictive model to the parameters of different local loss functions. This
‘feature-based parameterization’ gives us the best of both worlds—we retain the simplicity of learning local loss
functions while still being able to generalize across different decision-making instances. In addition to increasing
efficiency, this reparameterization also ensures that the learned loss functions are \textit{Fisher Consistent}---a
fundamental theoretical property that ensures that, in the limit of infinite data and model capacity, optimizing for the
loss function leads to optimal decision-making. Past methods for learning loss functions do not satisfy even this basic
theoretical property!

The second challenge with learning loss functions is that it presents a chicken-and-egg problem---to obtain the
distribution of predictions over which the learned loss function must accurately approximate the true decision quality,
a predictive model is required, yet to obtain such a model, a loss function is needed to train it. To address this,
\citet{shah2022decision} use a simplification we call the ``localness of prediction'', i.e., they assume that
predictions will be `close' to the true labels, and generate candidate predictions by adding random noise to the true
labels. However, this doesn't take into account the kinds of predictions that models actually generate and, as a result,
can lead to the loss functions being optimized for unrealistic predictions. In contrast, \citet{lawless2022note} and
\citet{chung2022decision} use a predictive model trained using MSE to produce a single representative sample that is
used to construct a simple loss function. However, this single sample is not sufficient to learn complex loss functions.
We explicitly formulate the goal of sampling a \textit{distribution} of realistic model predictions, and introduce a
technique that we call \textit{model-based sampling} to efficiently generate such samples.
Because these interventions allow us to move away from localness-based simplifications, we call our loss functions
`Efficient Global Losses' or \textit{EGLs} \eagle{}.

In addition to our theoretical Fisher Consistency results (\cref{sec:fisher}), we show the merits of EGLs empirically by
comparing them to past work on four domains (\cref{sec:experiments}). First, we show that in one of our key domains,
model-based sampling is essential for good performance, with EGLs outperforming even the best baseline by nearly 200\%.
The key characteristic of this domain is that it breaks the localness assumption (\cref{sec:localness}), which causes
past methods to fail catastrophically. Second, we show that EGLs achieve state-of-the-art performance on the remaining
three domains from the literature, \textit{despite} having an order of magnitude fewer samples than comparable methods
in two out of three of them. This improvement in sample efficiency translates to a reduction in the time taken to learn
task-specific loss functions, resulting in  an order-of-magnitude speed-up. All-in-all, we believe that these
improvements bring DFL one step closer to being accessible in practice.


\section{Background}
\subsection{Predict-then-Optimize} \label{sec:pto} There are two steps in \pto{} (PtO). In the \textit{Predict} step, a
learned predictive model $\predmodel$ is used to make predictions about uncertain quantities $[\pred_1, \ldots, \pred_N]
= [\predmodel(\feature_1), \ldots, \predmodel(\feature_N)]$ based on some features $[\feature_1, \ldots, \feature_N]$.
Next, in the \textit{Optimize} step, these predictions are aggregated as $\preds = [\pred_1, \ldots, \pred_N]$ and used
to parameterize an optimization problem $\optfn(\preds)$:
\begin{align*}
    \optfn(\preds) = \argmax_{\decision} \hspace{2mm} \obj(\decision; \preds), \quad
     s.t. \hspace{2mm} \decision \in \bm{\Omega}
\end{align*}
where $\obj$ is the objective and $\bm{\Omega} \subseteq \mathbb{R}^{\dim(\decision)}$ is the feasible region. The
solution to this optimization task $\decision = \optfn(\preds)$ provides the optimal decision for the predictions
$\preds$. We call a full set of inputs $\preds$ or $\labels = [\singlelabel_1, \ldots, \singlelabel_N]$ to the
optimization problem an \textit{instance} of the decision-making task. 

However, the optimal decision $\optfn(\preds)$ for the predictions $\preds$ may not be optimal for the true labels
$\labels$. To evaluate the \textit{Decision Quality} (DQ) of a set of predictions $\preds$, we measure how well the
decisions they induce $\optfn(\preds)$ perform on the set of \textit{true labels $\labels$} with respect to the
objective function $\obj$:
\begin{align}
    DQ(\preds, \labels) = \obj(\optfn(\preds); \labels) \label{eqn:dq}
\end{align}

The central question of \pto{}, then, is about how to learn predictive models $\predmodel$ that have high DQ. When
models are trained in a task-agnostic manner, e.g., to minimize Mean Squared Error (MSE), there can be a mismatch
between predictive accuracy and $DQ$, and past work (see \cref{sec:related}) has shown that the structure of the
optimization problem $\optfn$ can be used to learn predictive models with better $DQ$. We refer to this broad class of
methods for tailoring predictive models to decision-making tasks as Decision-Focused Learning (DFL) and describe one
recent approach below.

\subsection{Task-Specific Loss Functions} \label{sec:lodls} Multiple authors have suggested learning task-specific loss
functions for DFL~\citep{chung2022decision,lawless2022note,shah2022decision}. These approaches add learnable parameters
to standard loss functions (e.g., MSE) and tune them, such that the resulting loss functions approximate the `regret' in
DQ for `typical' predictions. Concretely, for the distribution over predictions $\preds = [\predmodel(\feature_1),
\ldots, \predmodel(\feature_N)]$ of the model $\predmodel$, the goal is to choose a loss function $L_{\bm{\phi}}$ with
parameters $\bm{\phi}$ such that:
\begin{gather*}
    \bm{\phi}^* = \argmin_{\bm{\phi}} \mathbb{E}_{\features, \labels, \preds} \left[\big(L_{\bm{\phi}}(\preds, \labels) - \dqregret(\preds, \labels) \big)^2 \right]\\
    \text{where}\quad \dqregret(\preds, \labels) \equiv DQ(\labels, \labels) - DQ(\preds, \labels) \numberthis \label{eqn:learnedloss}
\end{gather*}
where $DQ$ is defined in \cref{eqn:dq}. Note here that the first term in $\dqregret$ is a constant w.r.t. $\hat{y}$, so
minimizing $\dqregret$ is equivalent to maximizing $DQ$. Adding the $DQ(\labels, \labels)$ term, however, makes
$\dqregret$ behave more like a loss function---a minimization objective with a minimum value of 0 at $\preds = \labels$.
As a result, parameterized versions of simple loss functions can learn the structure of $\dqregret$ (and thus $DQ$).

A meta-algorithm for learning predictive models $\predmodel$ using task-specific loss functions is as follows:
\begin{enumerate}[leftmargin=1.2em,itemsep=0em,topsep=0.2em]
    \item \textbf{Sampling $\samples$:} Generate $K$ candidate predictions $\samples^k = [\singlelabel_1 \pm \epsilon_1,
    \ldots, \singlelabel_N \pm \epsilon_N]$, e.g., by adding Gaussian noise $\epsilon_n \sim \mathcal{N}(0, \sigma)$ to
    each true label $\singlelabel_n$. This strategy is motivated by the `localness of predictions' assumption, i.e.,
    that predictions will be close to the true labels.
    \item \textbf{Generating dataset:} Run a optimization solver on the sampled predictions $\samples$ to get the
    corresponding decision quality values $\dqregret(\samples, \labels)$. This results in a dataset of the form
    $[(\samples^1, \dqregret^1), \ldots, (\samples^K, \dqregret^K)]$ for each instance $\labels$ in the training and
    validation set.
    \item \textbf{Learning Loss Function(s):} Learn loss function(s) that minimize the MSE on the dataset from Step 2.
    \citet{lawless2022note} and \citet{chung2022decision} re-weight the MSE loss \textit{for each instance}:
    \begin{gather*}
        \text{L\&Z}^{\labels}_{w}(\preds) = w^{\labels} \cdot ||\preds - \labels||_2^2  \numberthis \label{eqn:landz}
    \end{gather*}
    \citet{shah2022decision} propose 2 families of loss functions, which they call `Locally Optimized Decision Losses'
    (LODLs). The first adds learnable weights to MSE \textit{for each prediction} that comprises the instance $\preds$.
    The second is more general---an arbitrary quadratic function of the predictions that comprise $\preds$, where the
    learned parameters are the coefficients of each polynomial term.
    \begin{align*}
        \text{WMSE}^{\labels}_{\bm{w}}(\preds) &= {\textstyle \sum_{n = 1}^{N}} w^{\labels}_n \cdot (\pred_n - \singlelabel_n)^2 \\
        \text{Quadratic}^{\labels}_{H}(\preds) &= (\preds - \labels)^T H^{\labels} (\preds - \labels) \numberthis \label{eqn:lodl}
    \end{align*}
    The parameters for these losses $w > w_{\min} > 0$ and $H = L^TL + w_{\min} \cdot I$ are constrained to ensure that
    the learned loss is convex. They also propose `Directed' variants of each loss in which the parameters to be learned
    are different based on whether $(\pred - \singlelabel) > 0$ or not. These parameters are then learned for every
    instance $\labels$, e.g., using gradient descent.
    \item \textbf{Learning predictive model $\predmodel$:} Train the predictive model $\predmodel$ on the loss functions
    learned in the previous step, e.g., a random forest~\citep{chung2022decision}, a neural
    network~\citep{shah2022decision}, or a linear model~\citep{lawless2022note}.
\end{enumerate}
In this paper, we propose two modifications to the meta-algorithm above. We modify Step 1 in \cref{sec:modelbased} and
Step 3 in \cref{sec:featurebased}. Given that these contributions help overcome the challenges associated with the
losses being "local", we call our new method \textit{EGL} \eagle{} (Efficient Global Losses).


\section{Related Work}\label{sec:related}
In \cref{sec:lodls}, we contextualize why the task-specific loss function approaches of
\citet{chung2022decision,lawless2022note,shah2022decision} make sense, and use the intuition to scale model-based
sampling and better empirically evaluate the utility of different aspects of learned losses.

In addition to learning loss functions, there are alternative approaches to DFL. The most common of these involves
optimizing for the decision quality directly by modeling the \pto{} task end-to-end and differentiating through the
optimization problem~\citep{agrawal2019differentiable,amos2018differentiable,donti2017task}. However, discrete
optimization problems do not have informative gradients, and so cannot be used as-is in an end-to-end pipeline. To
address this issue, most past work either constructs task-dependent surrogate problems that \textit{do} have informative
gradients for learning predictive
models~\citep{ferber2020mipaal,mensch2018differentiable,tschiatschek2018differentiable,wilder2019melding,wilder2019end}
or propose loss functions for specific classes of optimization problems (e.g. with linear
objectives~\citep{elmachtoub2021smart,mulamba2021contrastive}). However, the difficulty of finding good task-specific
relaxations for arbitrary optimization problems limits the adoption of these techniques in practice. On the other hand,
learning a loss function is more widely applicable as it applies to \textit{any optimization problem} and only requires
access to a solver for the optimization problem; as a result, we focus on improving this approach in this paper.


\section{Part One: Feature-based Parameterization}\label{sec:featurebased} The challenge of learning the
$\dqregret(\preds, \labels)$ function is that it is as hard as learning a closed-form solution to an arbitrary
optimization problem. To make the learning problem easier, past approaches typically learn multiple local
approximations, e.g., $\dqregret^i(\preds_i)$ \textit{for every decision-making instance} $\labels_i$ in the training
and validation set. This simplification trades off the complexity of learning a single complex $\dqregret$ for the cost
of learning many simpler $\dqregret^i$ functions:
\begin{align*}
    \underbrace{\dqregret(\preds, \labels)}_{\text{complex}} \approx \underbrace{\dqregret^i(\preds_i)}_{\text{simple}}, \; \forall \preds_i \approx \labels_i + \epsilon, \; \forall i \in [N]
\end{align*}
However, this is problematic because calculating $\dqregret$ for each of the sampled predictions $\samples$ is the most
expensive step in learning task-focused loss functions (see \cref{sec:timing}). Specifically, learning each
$\dqregret^i$ can require as many as $\Omega(\dim(\labels)^2)$ samples (for the `DirectedQuadratic' loss function from
\citet{shah2022decision}), leading to the need for $N \cdot \dim(\labels)^2$ samples overall. As a result, this approach
does not scale as the number of instances ($i \in [N]$) in the dataset or the size of the optimization problem
($\dim(\labels)$) increases.

To make learning loss functions more scalable, we have to find a way to trade-off the hardness of learning a global
approximation against the cost of learning a local approximation. The standard machine learning approach for this is to
replace learning $\dqregret^i$ with a loss $L$ that allows pooling samples across different instances $i \in [N]$ and
predictions $\dim(\labels)$. However, in doing this, two important design choices have to be made---(1) What should be
the functional form of $L$?, and (2) What should the inputs to $L$ be?

For the first, we could set $L$ to simply be some neural network model and try to learn some sort of global
approximation to $\dqregret$. However, in addition to the fact that this does not address the ``hardness'' of learning a
global approximation, \citet{shah2022decision} show that the lack of convexity of the learned neural networks can
sometimes lead to catastrophic failures, i.e., models that perform worse than random guessing. Building on this insight,
we propose using a neural network to instead learn a mapping to the \textit{parameters of a convex-by-construction loss
function family}. That way, regardless of the learned parameters, we ensure that $\preds^* = \labels$ is a minimizer of
the resulting loss function and so the learned loss is guaranteed to never fail catastrophically. 

For the second, we propose using the features $x_i$ associated with the prediction $\singlelabel_i$ as the input to our
learned loss. To ensure that $\preds^* = \labels$ regardless of how well we model $\dqregret$, we show in
\cref{sec:fisher} that the loss functions must have the same set of parameters for a given set of features $\features$
in order to be Fisher Consistent (discussed in more detail below).

Putting both of these together, EGLs \eagle{} learn a mapping $P_{\bm{\psi}}(\feature)$ from the features $\feature$ of
any prediction in the dataset to the corresponding parameters in convex-by-construction LODL loss families as follows:
\begin{itemize}[leftmargin=1.2em,itemsep=0em,topsep=0.2em]
    \item \textbf{WeightedMSE:} We learn a mapping $P_{\bm{\psi}}: \feature \to w$ from the features $\feature$ of a
    optimization parameter $\singlelabel$ to its associated `weight' $w$. Intuitively, the weight $w$ encodes how
    important a given prediction is, and so EGLs learn to predict the impact of different predictions' errors on
    $\dqregret$. 
    \item \textbf{Quadratic:} For every pair of predictions $\pred_i$ and $\pred_j$, we learn a mapping $P_{\bm{\psi}}:
    (\feature_i, \feature_j) \to L_{ij}$ where $L = [[L_{ij}]]$ is the matrix that parameterizes the loss function (see
    \cref{eqn:lodl}).
    \item \textbf{Directed Variants:} Instead of learning a mapping from the features $\feature$ to a single parameter,
    we instead learn a mapping from $\feature \to [w^+, w^-]$ for `Directed WeightedMSE' and $(\feature_i, \feature_j)
    \to [L^{++}, L^{\pm}, L^{-+}, L^{--}]$ for `Directed Quadratic'.
\end{itemize}
We then optimize for the optimal parameters $\bm{\psi}^*$ of the learned losses along the lines of past work:
\begin{gather*}
    \bm{\psi}^* = \argmin_{\bm{\psi}} \mathbb{E}_{\features, \labels, \preds} \left[\big(L_{P_{\bm{\psi}}(\features)}(\preds, \labels) - \dqregret(\preds, \labels) \big)^2 \right]
\end{gather*}
where $L_{\bm{\phi}} = L_{P_{\bm{\psi}}(\features)}$ is the learned loss function. For our experiments, the model
$P_{\bm{\psi}}$ is a 4-layer feedforward neural network with a hidden dimension of 500 trained using gradient descent.
Given that we learn a mapping from the features of a given prediction $\feature$ to the corresponding loss function
parameter(s), we call our resulting approach \emph{feature-based parameterization (FBP)}.

\subsection{Fisher Consistency} \label{sec:fisher} One desirable property of a \pto{} loss function is that, in the
limit of infinite data and model capacity, the optimal prediction induced by the loss also minimizes the decision
quality regret. If this is true, we say that loss $\ell$ is ``Fisher Consistent" w.r.t. the decision quality regret
$\dqregret$~\citep{elmachtoub2021smart}.
\begin{definition}[Fisher Consistency]
A loss $\ell(\preds, \labels)$ is said to be Fisher Consistent with respect to the decision quality regret $\dqregret$
if the set of minimizers of the loss function $\preds^*(\features) = \argmin_{\preds} \mathbb{E}_{\labels |
\features}[\ell(\preds, \labels)]$ also minimize $\dqregret$ for all possible distributions $P(\features, \labels)$.
\end{definition}
However, the methods proposed in past work do not satisfy this property even for the simplest of cases---in which the
objective $\obj$ of the optimization problem $\optfn$ is linear. Concretely:
\begin{theorem}\label{thm:lodlfisher} Weighting-the-MSE losses are not Fisher Consistent for \pto{} problems in which
the optimization function $\optfn$ has a linear objective.
\end{theorem}
\begin{proof}
    \label{sec:example}
    We show a proof by counter-example below. Consider a PtO problem in which the goal is to (a) predict the utility of
    a resource for two individuals (say, $A$ and $B$), and then (b) give the resource to the individual with higher
    utility. Consider the utilities to be drawn from:
    \begin{align*}
        \labels = (\singlelabel_A, \singlelabel_B) = 
        \scalebox{0.9}{%
        $\begin{cases}
            {\color{figureblue} (0, 0.55),\quad\text{ with probability }0.5}\\
            {\color{figureorange} (1, 0.55),\quad\text{ with probability }0.5}
        \end{cases}$%
        }
    \end{align*}
    The optimal decision, then, is to give the resource to individual B because $\mathbb{E}[\pred_B] = 0.55 >
    \mathbb{E}[\pred_A] = 0.5$.

    To learn a predictive model using a Weighting-the-MSE loss, we follow the meta-algorithm in \cref{sec:lodls}:
    \begin{enumerate}[leftmargin=1.4em,itemsep=0em,topsep=0.2em]
    \item We sample K = 25 points in the ``neighborhood'' of each of the true labels {\color{figureblue} $(0, 0.55)$}
    and {\color{figureorange} $(1, 0.55)$}. For simplicity, we add uniform-random noise $\epsilon^k \sim U[-1,1]$ only
    to $\singlelabel_A$ to get $\samples^k = (\singlelabel_A \pm \epsilon^k, \singlelabel_B)$.
    \item We calculate $\dqregret$ for each sample. We plot $\dqregret$ vs $\pred_A$ in \cref{fig:example} (given
    $\pred_B$ is fixed).
    \begin{figure}
        \centering
        \centering
        \includegraphics[width=0.8\linewidth]{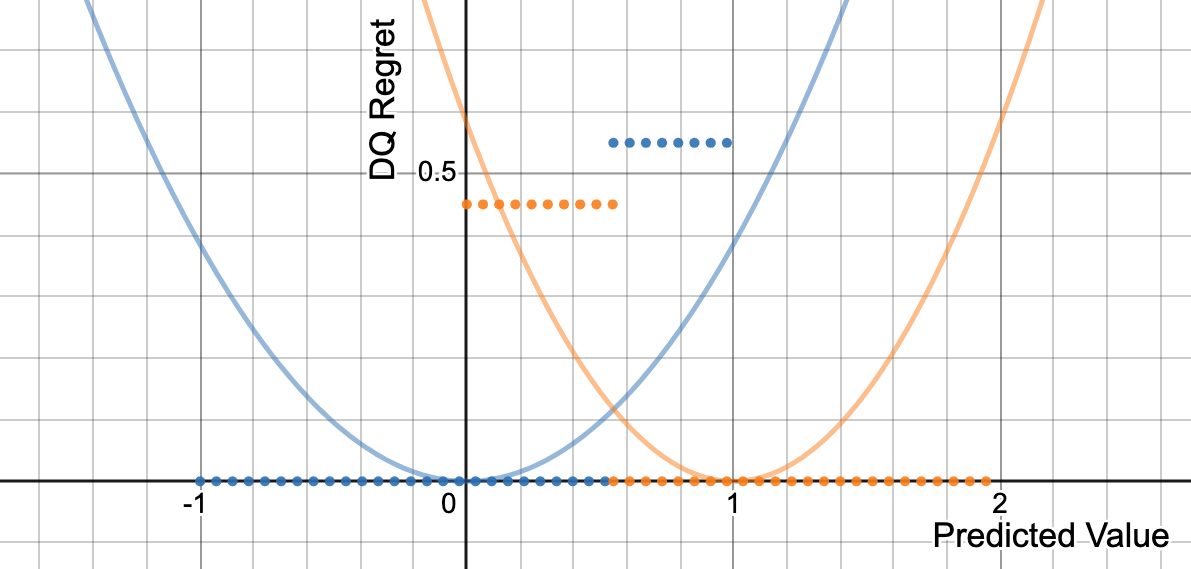}
        \caption{\textbf{A plot of $\dqregret$ vs. $\pred_A$} (for fixed $\pred_B$). The samples from Step 1 correspond
        to the {\color{figureblue}blue dots} for the {\color{figureblue}$(0, 0.55)$ instance} and the
        {\color{figureorange}orange dots} for the {\color{figureorange}$(1, 0.55)$ instance}. We also plot the learned
        weighted MSE loss for each instance using solid lines in their corresponding colors.}
        \label{fig:example}
    \end{figure}
    \item Based on this dataset, we fit a loss of the form given by \cref{eqn:landz}. Because there is only one weight
    being learned here, this can be seen as either a WeightedMSE LODL from \citet{shah2022decision} or a reweighted
    task-loss from ~\citet{lawless2022note}, as seen in \cref{fig:example}.
    \item Finally, we estimate a predictive model based on this loss. The optimal prediction given this loss is
    $\preds^* = (\pred^*_A, \pred^*_B) \approx (0.602,0.55)$ (see \cref{sec:calculation} for details).
    \end{enumerate}
    Under this predictive model, the optimal decision would be to give the resource to individual $A$ because they have
    higher predicted utility, i.e., $\pred^*_A > \pred^*_B$. \textit{But this is suboptimal!}
\end{proof}
More generally, because ``WeightedMSE'' can be seen as a special case of the other methods in \citet{shah2022decision},
none of the loss functions in past work are Fisher Consistent! To gain intuition for why this happens, let us analyze
the predictions that are induced by weighting-the-MSE type losses:
\begin{lemma}\label{lemma:weights} The optimal prediction $\pred^*(\feature)$ for some feature $\feature$ given a
weighted MSE loss function with weights $w^{\labels}$ associated with the label $\singlelabel \in \labels$ is
$\pred^*(\feature) = \frac{\mathbb{E}_{\labels | \feature}[w^{\labels} \cdot y]}{\mathbb{E}_{\labels |
\feature}[w^{\labels}]}$, given infinite model capacity.
\end{lemma}
The proof is presented in \cref{sec:lemmaproof}. In their paper, \citet{elmachtoub2021smart} show that the optimal
prediction that minimizes $\dqregret$ is $\pred^*(\feature) = \mathbb{E}_{\labels | \feature}[\singlelabel]$, i.e., the
optimal prediction should depend only on the value of the labels corresponding to that feature. However,
$\pred^*(\feature) = \frac{\mathbb{E}_{\labels | \feature}[w^{\labels} \cdot \singlelabel]}{\mathbb{E}_{\labels |
\feature}[w^{\labels}]}$ in weighted MSE losses, and so the optimal prediction depends on not only the labels but also
the \textit{weight} associated with them. While this is not a problem by itself, the weight learned for a label
$\singlelabel \in \labels$ is dependent on not only the label itself but also the \textit{other labels
$\labels_{-\singlelabel}$ in that instance}. In the context of the counter-example in the proof of
\cref{thm:lodlfisher}, the weight associated with individual $A$ is dependent on the utility of individual $B$ (via
$\dqregret$). As a result, it is possible to create a distribution $P(\features, \labels)$ for which such losses will
not be Fisher Consistent. However, this is not true for WeightedMSE with FBP!

\begin{theorem}
    WeightedMSE with FBP is Fisher Consistent for \pto{} problems in which the optimization function $\optfn$ has a
    linear objective.
\end{theorem}
\begin{proof}
    In WeightedMSE with FBP, the weights associated with some feature $\feature$ are not independently learned for each
    instance $\labels$ but are instead a \textit{function of the features $\feature$}. As a result, the weight
    $w^{\labels}$ associated with that feature is the same across all instances, i.e., $w^{\labels} = w(\feature)$,
    $\forall \labels$. Plugging that into the equation from \cref{lemma:weights}:
    \begin{align*}
        \pred^*(\feature) &= \frac{\mathbb{E}_{\labels | \feature}[w^{\labels} \cdot y]}{\mathbb{E}_{\labels | \feature}[w^{\labels}]} = \frac{w(\feature)\cdot \mathbb{E}_{\labels | \feature}[y]}{w(\feature)} = \mathbb{E}_{\labels | \feature}[y]
    \end{align*}
    which is a minimizer of $\dqregret$~\citep{elmachtoub2021smart}. Thus, WeightedMSE with FBP is Fisher Consistent.
\end{proof}

\section{Part Two: Model-based Sampling}\label{sec:modelbased} Loss functions serve to give feedback to the model.
However, to make them easier to learn, their expressivity is often limited. As a result, they cannot estimate
$\dqregret$ accurately for \textit{all} possible predictions but must instead concentrate on a subset of ``realistic
predictions'' for which the predictive model $\predmodel$ will require feedback during training. However, there is a
chicken-and-egg problem in learning loss functions on realistic predictions---a model is needed to generate such
predictions, but creating such a model would in turn require its own loss function to train on.

Past work has made the assumption that the predictions $\preds$ will be close to the actual labels $\labels$ to
efficiently generate potential predictions $\samples$. However, if the assumption does not hold, Gaussian sampling may
not yield good results, as seen in \cref{sec:localness}. Instead, in this paper, we propose an alternative:
\emph{model-based sampling (MBS)}. Here, to generate a \textit{distribution} of potential predictions using this
approach, we train a predictive model $\predmodel$ on a standard loss function (e.g., MSE). Then, at equally spaced
intervals during the training process, we use the intermediate model to generate predictions for each problem instance
in the dataset. These form the set of \textit{potential predictions} $\samples$ based on which we create the dataset and
learn loss functions.
The hyperparameters associated with this approach are:
\begin{itemize}[leftmargin=1.2em,itemsep=0em,topsep=0.2em]
\item \textbf{Number of Models:} Instead of sampling predictions from just one model, we can train multiple models to
increase the diversity of the generated predictions. In our experiments, we choose from $\{1, 5, 10\}$ predictive
models.
\item \textbf{LR} and \textbf{Number of Training Steps:} The learning rates are chosen from $\{10^{-6}, 10^{-5}, \ldots,
1\}$ with a possible cyclic schedule~\citep{smith2017cyclical}. We use a maximum of $50000$ updates across all the
models.
\end{itemize}
Empirically, we find that a high learning rate and large number of models works best. Both of these choices increase the
diversity of the generated samples and help create a richer dataset for learning loss functions.

\subsection{Localness of Predictions} \label{sec:localness} To illustrate the utility of model-based sampling, we
analyze the Cubic Top-K domain proposed by~\citet{shah2022decision}.
The goal in this domain is to fit a linear model to approximate a more complex cubic relationship between $\feature$ and
$\singlelabel$ (\cref{fig:cubic}, \cref{sec:localnessapp}). This could be motivated by
explainablity~\cite{rudin2019stop,pmlr-v84-hughes18a}, data efficiency, or simplicity. The localness assumption breaks
here because it is not possible for linear models (or low-capacity models more generally) to closely approximate the
true labels of a more complex data-generating process. The objective of the learned loss function, then, is to provide
information about \textit{what kind of suboptimal predictions are better than others}.

Learned losses accomplish this by first generating plausible predictions and then learning how different sorts of errors
change the decision quality. In this domain, the decision quality is solely determined by the point with the highest
predicted utility. As can be seen in \cref{fig:cubic} (\cref{sec:localnessapp}), the highest values given $\feature \sim
U[-1, 1]$ are either at $x = -0.5$ or $x = 1$. In fact, because the function is flatter around $x = -0.5$, there are
more likely to be large values there. When Gaussian sampling~\cite{shah2022decision} is used to generate candidate
predictions, the highest sampled values are also more likely to be at $x = -0.5$ because the added noise has a mean of
zero. However, \textit{a linear model cannot have a maximum value at $x = -0.5$}, only $x \in \{-1, 1\}$. As a result,
the loss functions learned based on the samples from this method focus on the wrong subset of labels and lead to poor
downstream performance. On the other hand, the candidate predictions generated by model-based sampling are the outputs
of linear models, allowing the loss functions to take this into account. We visualize this phenomenon in
\cref{fig:localness}.

In their paper, \citet{shah2022decision} propose a set of ``directed models'' to make LODLs perform well in this domain.
However, these models only learn useful loss functions because the value of the label at $x = 1$ is \textit{slightly
higher} than the value at $x = -0.5$. To show this, we create a variant of this domain called ``(Hard) Cubic Top-K'' in
which  $y_{x = -0.5} > y_{x = 1}$. Then, in \cref{tab:overall}, we see that even the ``directed'' LODLs fail
catastrophically in this domain, while the loss functions learned with model-based sampling perform well.


\begin{table*}[h]
\centering
\caption{\textbf{Overall Results.} The entries represent the Mean Normalized Test DQ $\pm$ SEM. Methods that are not
applicable to specific domains are denoted by `-'. Bolded values represent the set of methods that outperform the others
by a statistically significant margin (p-value < 0.05). Our method achieves state-of-the-art performance without any
handcrafting and an order of magnitude fewer samples.}
\label{tab:overall}
\resizebox{0.8\linewidth}{!}{%
\begin{tabular}{cccccc}
\toprule
\multirow{3}[4]{*}{Category}&\multirow{3}[4]{*}{Method} & \makecell{New Domain} & \multicolumn{3}{c}{Domains from the
Literature}\\
 \cmidrule(lr){3-3} \cmidrule(lr){4-6} && \makecell{(Hard) Cubic\\Top-K} & \makecell{Cubic\\Top-K} &
 \makecell{Web\\Advertising} & \makecell{Portfolio\\Optimization} \\
\midrule
2-Stage & MSE & -0.65 $\pm$ 0.04 & -0.50 $\pm$ 0.06 & 0.60 $\pm$ 0.04 & 0.04 $\pm$ 0.00 \\
\arrayrulecolor{black!40}\midrule \multirow{4}[5]{*}{\makecell{Expert-crafted\\Surrogates}} & SPO+ & -0.68 $\pm$ 0.00 &
\textbf{0.96 $\pm$ 0.00} & - & - \\ & \makecell{Entropy-Regularized\\Top-K} & 0.24 $\pm$ 0.08 & \textbf{0.96 $\pm$ 0.00}
& - & - \\ & Multilinear Relaxation & - & - & 0.74 $\pm$ 0.01 & - \\ & Differentiable QP & - & - & - & 0.141 $\pm$ 0.003
\\
\arrayrulecolor{black!40}\midrule \multirow{4}{*}{\makecell{Learned Losses}} & L\&Z (1 Sample) & -0.68 $\pm$ 0.00 &
-0.96 $\pm$ 0.00 & 0.65 $\pm$ 0.02 & 0.133 $\pm$ 0.005 \\ &\makecell{LODL (32 Samples)} & -0.68 $\pm$ 0.00 & -0.38 $\pm$
0.29 & 0.84 $\pm$ 0.04 & \textbf{0.146 $\pm$ 0.003}\\ &\makecell{LODL (2048 Samples)} & -0.67 $\pm$ 0.01 & \textbf{0.96
$\pm$ 0.00} & \textbf{0.93 $\pm$ 0.01} & \textbf{0.154 $\pm$ 0.005} \\ &\makecell{EGL \eagle{} (32 Samples)} &
\textbf{0.69 $\pm$ 0.00} & \textbf{0.96 $\pm$ 0.00} & \textbf{0.95 $\pm$ 0.01} & \textbf{0.153 $\pm$ 0.004} \\
\arrayrulecolor{black}\bottomrule \end{tabular}%
}
\end{table*}

\section{Experiments} \label{sec:experiments} In this section, we validate EGLs empirically on four domains from the
literature. We provide brief descriptions of the domains below but refer the reader to the corresponding papers for more
details.

\vspace{0.2em}\noindent \textbf{Cubic Top-K~\cite{shah2022decision}}\hspace{0.5em} Learn a model whose top-k predictions
have high corresponding true labels.
\begin{itemize}[leftmargin=1em,itemsep=0em,topsep=0.2em]
\item \textit{Predict:} Predict resource $n$'s utility $\pred_n$ using a linear model with feature $x_n \sim U[-1, 1]$.
The true utility is $y_n = 10x_n^3 - 6.5x_n$ for the standard version of the domain and $y_n = 10x_n^3 -
\mathbf{7.5}x_n$ for the `hard' version. The predictive model is linear, i.e., $\predmodel(x) = mx + c$.
\item \textit{Optimize:} Out of $N=50$ resources, choose the top $K = 1$ resources with highest utility.
\end{itemize}

\noindent \textbf{Web Advertising~\cite{wilder2019melding}}\hspace{0.5em} The objective of this domain is to predict the
Click-Through-Rates (CTRs) of different (user, website) pairs such that good websites to advertise on are chosen.
\begin{itemize}[leftmargin=1em,itemsep=0em,topsep=0.2em]
\item \textit{Predict:} Predict the CTRs $\preds_m$ for $N = 10$ fixed users on $M = 5$ websites using the website's
features $\feature_m$. The features for each website are obtained by multiplying the true CTRs $\labels_m$ from the
Yahoo! Webscope Dataset~\citep{webscope} with a random $N \times N$ matrix $A$, resulting in $\features_m = A
\labels_m$. The predictive model $\predmodel$ is a 2-layer feedforward network with a hidden dimension of 500.
\item \textit{Optimize:} Choose which $K = 2$ websites to advertise on such that the expected number of users who click
on an advertisement at least once (according to the CTR matrix) is maximized. The objective to be maximized is
$\optfn(\preds) = \argmax_{\decision} \sum_{j = 0}^N (1 - \prod_{i = 0}^M (1 - z_i \cdot \pred_{ij}))$, where $z_i$ can
be either 0 or 1. This is a submodular maximization problem.
\end{itemize}

\noindent \textbf{Portfolio Optimization~\cite{donti2017task}}\hspace{1em} Based on the Markovitz
model~\citep{markowitz2000mean}, the aim is to predict future stock prices in order to create a portfolio that has high
returns but low risk.
\begin{itemize}[leftmargin=1em,itemsep=0em,topsep=0.2em]
\item \textit{Predict:} Predict the future stock price $y_n$ for each stock $n$ using its historical data $\feature_n$.
The historical data includes information on 50 stocks obtained from the QuandlWIKI dataset~\citep{QuandlWIKI}. The model
$\predmodel$ is a 2-layer feedforward network with a hidden dimension of 500.
\item \textit{Optimize:} Choose a distribution $\decision$ over stocks to maximize $\decision^T\labels - \lambda \cdot
\preds^TQ\preds$ based on a known correlation matrix $Q$ of stock prices. Here, $\lambda = 0.001$ represents the
constant for risk aversion.
\end{itemize}

\vspace{0.5em}\noindent For each set of experiments, we run 10 experiments with different train-test splits, and
randomized initializations of the predictive model and loss function parameters. Details of the computational
resources~\cite{OhioSupercomputerCenter1987} and hyperparameter optimization used are given in \cref{sec:extresults}.

For all of these domains, the metric of interest is the decision quality achieved by the predictive model $\predmodel$
on the hold-out test set when trained with the loss function in question. However, given that the scales of the decision
quality for each domain vary widely, we linearly re-scale the value such that 0 corresponds to the DQ of making
predictions uniformly at random $\pred = \epsilon \sim U[0, 1]$ and 1 corresponds to making perfect predictions $\pred =
\singlelabel$. Concretely:
$$\text{Normalized DQ}(\preds, \labels) = \frac{DQ(\preds, \labels) - DQ(\bm{\epsilon}, \labels)}{DQ(\labels, \labels) -
DQ(\bm{\epsilon}, \labels)}$$

\subsection{Overall Results}\label{sec:results} We compare our approach against the following baselines from the
literature in \cref{tab:overall}:
\begin{itemize}[leftmargin=1.2em,itemsep=0em,topsep=0.2em]
\item \textbf{MSE:} A standard regression loss.
\item \textbf{Expert-crafted Surrogates:} The end-to-end approaches described in \cref{sec:related} that require
handcrafting differentiable surrogate optimization problems for each domain
separately~\cite{donti2017task,wang2020automatically,elmachtoub2021smart,xie2020differentiable,wilder2019end}.
\item \textbf{L\&Z:} \citet{lawless2022note}'s approach for learning losses (equivalent to \citet{chung2022decision}).
\item \textbf{LODL:} \citet{shah2022decision}'s approach for learning loss functions. Trained using 32 and 2048 samples.
\end{itemize}

\begin{table*}[t]
\centering
\caption{\textbf{Comparison to LODLs.} MBS $\to$ Model-based Sampling, FBP $\to$ Feature-based Parameterization, and EGL
= LODL + MBS + FBP. The entries represent the Mean Normalized Test DQ $\pm$ SEM. EGLs improve the DQ for almost every
choice of loss function family and domain.}
\label{tab:vslodl}
\resizebox{0.82\linewidth}{!}{%
\begin{tabular}{cccccc}
\toprule
 \multirow{3}[4]{*}{Domain} & \multirow{3}[4]{*}{Method} & \multicolumn{4}{c}{Normalized Test Decision Quality} \\ 
 \cmidrule(lr){3-6} && \makecell{Directed\\Quadratic} & \makecell{Directed\\WeightedMSE} & Quadratic & WeightedMSE \\
\arrayrulecolor{black}\midrule \multirow{4}[5]{*}{\makecell{Cubic\\Top-K}} & LODL & -0.38 $\pm$ 0.29 & -0.86 $\pm$ 0.10
& -0.76 $\pm$ 0.19 & -0.95 $\pm$ 0.01 \\ & \textcolor{black!40}{\makecell{LODL (2048 samples)}} &
\textcolor{black!40}{-0.94 $\pm$ 0.01} & \textcolor{black!40}{0.96 $\pm$ 0.00} & \textcolor{black!40}{-0.95 $\pm$ 0.01}
& \textcolor{black!40}{-0.96 $\pm$ 0.00} \\
\arrayrulecolor{black!40}\cmidrule(lr){2-6} & EGL (MBS) & \textbf{0.96 $\pm$ 0.00} & \textbf{0.96 $\pm$ 0.00} & 0.77
$\pm$ 0.13 & \textbf{0.77 $\pm$ 0.13} \\ & EGL (FBP) & 0.58 $\pm$ 0.26 & \textbf{0.96 $\pm$ 0.00} & -0.28 $\pm$ 0.21 &
-0.77 $\pm$ 0.11 \\ & EGL (Both) & \textbf{0.96 $\pm$ 0.00} & 0.77 $\pm$ 0.13 & \textbf{0.96 $\pm$ 0.00} & \textbf{0.77
$\pm$ 0.13} \\
\arrayrulecolor{black}\midrule \multirow{5}[6]{*}{\makecell{Web\\Advertising}} & LODL & 0.75 $\pm$ 0.05 & 0.72 $\pm$
0.03 & 0.84 $\pm$ 0.04 & 0.71 $\pm$ 0.03 \\ & \textcolor{black!40}{\makecell{LODL (2048 samples)}} &
\textcolor{black!40}{0.93 $\pm$ 0.01} & \textcolor{black!40}{0.84 $\pm$ 0.02} & \textcolor{black!40}{0.93 $\pm$ 0.01} &
\textcolor{black!40}{0.78 $\pm$ 0.03} \\
\arrayrulecolor{black!40}\cmidrule(lr){2-6} & EGL (MBS) & 0.86 $\pm$ 0.03 & \textbf{0.83 $\pm$ 0.03} & 0.78 $\pm$ 0.06 &
0.77 $\pm$ 0.04 \\ & EGL (FBP) & 0.93 $\pm$ 0.02 & 0.80 $\pm$ 0.03 & \textbf{0.92 $\pm$ 0.01} & 0.75 $\pm$ 0.04 \\ & EGL
(Both) & \textbf{0.95 $\pm$ 0.01} & 0.78 $\pm$ 0.06 & \textbf{0.92 $\pm$ 0.02} & \textbf{0.81 $\pm$ 0.04} \\
\arrayrulecolor{black}\midrule \multirow{5}[6]{*}{\makecell{Portfolio\\Optimization}} & LODL & \textbf{0.146 $\pm$
0.003} & 0.136 $\pm$ 0.003 & 0.145 $\pm$ 0.003 & 0.122 $\pm$ 0.003 \\ & \textcolor{black!40}{\makecell{LODL (2048
samples)}} & \textcolor{black!40}{0.154 $\pm$ 0.005} & \textcolor{black!40}{0.141 $\pm$ 0.004} &
\textcolor{black!40}{0.147 $\pm$ 0.004} & \textcolor{black!40}{0.113 $\pm$ 0.014} \\
\arrayrulecolor{black!40}\cmidrule(lr){2-6} & EGL (MBS) & 0.135 $\pm$ 0.011 & 0.138 $\pm$ 0.010 & 0.146 $\pm$ 0.015 &
0.108 $\pm$ 0.009 \\ & EGL (FBP) & 0.139 $\pm$ 0.005 & 0.141 $\pm$ 0.008 & \textbf{0.147 $\pm$ 0.008} & 0.136 $\pm$
0.004 \\ & EGL (Both) & 0.134 $\pm$ 0.013 & \textbf{0.127 $\pm$ 0.011} & 0.145 $\pm$ 0.011 & \textbf{0.153 $\pm$ 0.004}
\\
\arrayrulecolor{black}\bottomrule \end{tabular}%
}
\end{table*}

\begin{table}[h]
    \centering
    \caption{\textbf{Time taken to run the meta-algorithm} (\cref{sec:lodls}) for comparable WeightedMSE EGLs and LODLs
    on the Web Advertising domain. EGLs take only 6\% of the LODLs' time to train.}
    \resizebox{\linewidth}{!}{%
    \begin{tabular}{c c c} 
        \toprule
        \multirow{2}[3]{*}{\makecell{\textbf{Time Taken}\\\textcolor{black!40}{(in seconds)}}} &
        \multicolumn{2}{c}{\textbf{Method}}\\
        \cmidrule{2-3}
        & LODL & EGL \eagle{} \\
        \midrule
        \makecell{Samping $\samples$\\\textcolor{black!40}{(Step 1)}} & \makecell{0.18 $\pm$
        0.01\\\textcolor{black!40}{(Gaussian Sampling)}} & \makecell{0.48 $\pm$ 0.04\\\textcolor{black!40}{(MBS)}} \\
        \makecell{Generating Dataset\\\textcolor{black!40}{(Step 2)}} & \makecell{10376.65 $\pm$
        119.81\\\textcolor{black!40}{(2048 samples)}} & \makecell{200.43 $\pm$ 4.24\\\textcolor{black!40}{(32 samples)}}
        \\
        \makecell{Learning Losses\\\textcolor{black!40}{(Step 3)}} & \makecell{53.67 $\pm$
        1.84\\\textcolor{black!40}{(Separate Losses)}} & \makecell{445.67 $\pm$ 62.42\\\textcolor{black!40}{(FBP)}} \\
        \midrule
        \makecell{Total} & 10430.50 $\pm$ 131.66 & 646.58 $\pm$ 66.70\\
        \bottomrule
    \end{tabular}%
    }
    \label{tab:runtime}
\end{table}

\noindent \textbf{(Hard) Cubic Top-K}\hspace{1em}We empirically verify our analysis from~\cref{sec:localness} by testing
different baselines on our proposed ``hard'' top-k domain. In \cref{tab:overall}, we see that all our baselines perform
extremely poorly in this domain. Even the expert-crafted surrogate only achieves a DQ of $0.24$ while EGLs achieve the
best possible DQ of $0.69$; this corresponds to a gain of nearly 200\% for EGLs.

\noindent \textbf{Domains for the Literature}\hspace{1em}We find that our method reaches state-of-the-art performance in
all the domains from the literature. In fact, we see that EGLs achieve similar performance to LODLs with an order of
magnitude fewer samples in two out of three domains. In \cref{sec:timing} below, we see that this corresponds to
\textit{an order of magnitude speed-up over learning LODLs of similar quality!}

\subsection{Computational Complexity Experiments}\label{sec:timing} We saw in \cref{sec:results} that EGLs perform as
well as LODLs with an order of magnitude fewer samples. In \cref{tab:runtime}, we show how this increased sample
efficiency translates to differences in runtime. We see that, by far, most of the time in learning LODLs is spent in
step 2 of our meta-algorithm. As a result, despite the fact that EGLs take longer to perform steps 1 and 3, the increase
in sample efficiency results in an \textit{order-of-magnitude} speedup over LODLs.

\subsection{Ablation Study}\label{sec:lodlcomparison} In this section, we compare EGLs to their strongest competitor
from the literature, i.e., LODLs \cite{shah2022decision}. Specifically, we look at the low-sample regime---when 32
samples per instance are used to train both losses---and present our results in \cref{tab:vslodl}. We see that EGLs
improve the decision quality for almost every choice of loss function family and domain. We further analyze
\cref{tab:vslodl} below.

 \textbf{Feature-based Parameterization (FBP):} Given that this is the low-sample regime, `LODL + FBP' almost always
 does better than just LODL. These gains are especially apparent in cases where adding more samples would improve LODL
 performance---the ``Directed'' variants in the Cubic Top-K domain, and the ``Quadratic'' methods in the Web Advertising
 domain.

\textbf{Model-based Sampling (MBS):} This contribution is most useful in the Cubic Top-K domain, where the localness
assumption is broken. Interestingly, however, MBS also improves performance in the other two domains where the localness
assumption does not seem to be broken (\cref{tab:mseerror} in \cref{sec:localexp}). We hypothesize that MBS helps here
in two different ways:
\begin{enumerate}[leftmargin=1.2em,itemsep=0em,topsep=-0.5em]
\item \emph{Increasing effective sample efficiency:} We see that, in the cases where FBP helps most, the gains from MBS
stack with FBP. This suggests that MBS helps improve sample-efficiency. Our hypothesis is that model-based sampling
allows us to focus on predictions that would lead to a ``fork in the training trajectory'', leading to improved
performance with fewer samples.
\item \emph{Helping WeightedMSE models:} MBS also helps improve the \textit{worst-performing} WeightedMSE models in
these domains which, when combined with FBP, outperform even LODLs with 2048 samples. This suggests that MBS does more
than just increase sample efficiency. We hypothesize that MBS also reduces the search space by limiting the set of
samples $\samples$ to ``realistic predictions'', allowing even WeightedMSE models that have fewer parameters to perform
well in practice. 
\end{enumerate}

\vspace{0.5em}
\textbf{Portfolio Optimization:} The results for this domain don't follow the trends noted above because there is a
distribution shift between the validation and test sets in this domain (as the train/test/validation split is temporal
instead of i.i.d.).  In \cref{tab:portfolioopt} (\cref{sec:portfolioapp}), we see that EGLs outperform LODLs and follow
the trends noted above if we measure their performance on the validation set, which is closer in time to training (and
hence has less distribution shift). 

\section*{Acknowledgements}
Sanket Shah was supported by the ARO under Grant Number: W911NF-18-1-0208. The views and conclusions contained in this
document are those of the author s and should not be interpreted as representing the official policies, either expressed
or implied, of ARO or the U.S. Government. The U.S. Government is authorized to reproduce and distribute reprints for
Government purposes notwithstanding any copyright notation herein. Bryan Wilder was supported by NSF (Award Number
2229881). 

\bibliography{main}

\begin{thebibliography}{25}
\providecommand{\natexlab}[1]{#1}

\bibitem[{Agrawal et~al.(2019)Agrawal, Amos, Barratt, Boyd, Diamond, and
  Kolter}]{agrawal2019differentiable}
Agrawal, A.; Amos, B.; Barratt, S.; Boyd, S.; Diamond, S.; and Kolter, J.~Z.
  2019.
\newblock Differentiable convex optimization layers.
\newblock \emph{Advances in Neural Information Processing Systems}, 32.

\bibitem[{Amos et~al.(2018)Amos, Jimenez, Sacks, Boots, and
  Kolter}]{amos2018differentiable}
Amos, B.; Jimenez, I.; Sacks, J.; Boots, B.; and Kolter, J.~Z. 2018.
\newblock Differentiable MPC for End-to-end Planning and Control.
\newblock In \emph{Advances in Neural Information Processing Systems},
  volume~31.

\bibitem[{Bengio(1997)}]{bengio1997using}
Bengio, Y. 1997.
\newblock Using a financial training criterion rather than a prediction
  criterion.
\newblock \emph{International journal of neural systems}, 8(04): 433--443.

\bibitem[{Chung et~al.(2022)Chung, Rostami, Bastani, and
  Bastani}]{chung2022decision}
Chung, T.-H.; Rostami, V.; Bastani, H.; and Bastani, O. 2022.
\newblock Decision-Aware Learning for Optimizing Health Supply Chains.
\newblock \emph{arXiv preprint arXiv:2211.08507}.

\bibitem[{Donti, Amos, and Kolter(2017)}]{donti2017task}
Donti, P.; Amos, B.; and Kolter, J.~Z. 2017.
\newblock Task-based end-to-end model learning in stochastic optimization.
\newblock \emph{Advances in Neural Information Processing Systems}, 30.

\bibitem[{Elmachtoub and Grigas(2021)}]{elmachtoub2021smart}
Elmachtoub, A.~N.; and Grigas, P. 2021.
\newblock Smart “predict, then optimize”.
\newblock \emph{Management Science}.

\bibitem[{Ferber et~al.(2020)Ferber, Wilder, Dilkina, and
  Tambe}]{ferber2020mipaal}
Ferber, A.; Wilder, B.; Dilkina, B.; and Tambe, M. 2020.
\newblock {MIPaaL}: Mixed integer program as a layer.
\newblock In \emph{Proceedings of the AAAI Conference on Artificial
  Intelligence}, volume~34, 1504--1511.

\bibitem[{Hughes et~al.(2018)Hughes, Hope, Weiner, McCoy, Perlis, Sudderth, and
  Doshi-Velez}]{pmlr-v84-hughes18a}
Hughes, M.; Hope, G.; Weiner, L.; McCoy, T.; Perlis, R.; Sudderth, E.; and
  Doshi-Velez, F. 2018.
\newblock Semi-Supervised Prediction-Constrained Topic Models.
\newblock In \emph{Proceedings of the Twenty-First International Conference on
  Artificial Intelligence and Statistics}, 1067--1076.

\bibitem[{Lawless and Zhou(2022)}]{lawless2022note}
Lawless, C.; and Zhou, A. 2022.
\newblock A Note on Task-Aware Loss via Reweighing Prediction Loss by
  Decision-Regret.
\newblock \emph{arXiv preprint arXiv:2211.05116}.

\bibitem[{Markowitz and Todd(2000)}]{markowitz2000mean}
Markowitz, H.~M.; and Todd, G.~P. 2000.
\newblock \emph{Mean-variance analysis in portfolio choice and capital
  markets}.
\newblock John Wiley \& Sons.

\bibitem[{Mensch and Blondel(2018)}]{mensch2018differentiable}
Mensch, A.; and Blondel, M. 2018.
\newblock Differentiable dynamic programming for structured prediction and
  attention.
\newblock In \emph{International Conference on Machine Learning}, 3462--3471.
  PMLR.

\bibitem[{Mulamba et~al.(2021)Mulamba, Mandi, Diligenti, Lombardi, Bucarey, and
  Guns}]{mulamba2021contrastive}
Mulamba, M.; Mandi, J.; Diligenti, M.; Lombardi, M.; Bucarey, V.; and Guns, T.
  2021.
\newblock Contrastive Losses and Solution Caching for Predict-and-Optimize.
\newblock In \emph{Proceedings of the International Joint Conferences on
  Artificial Intelligence}.

\bibitem[{OSC(1987)}]{OhioSupercomputerCenter1987}
OSC. 1987.
\newblock Ohio Supercomputer Center.
\newblock \url{http://osc.edu/ark:/19495/f5s1ph73}.

\bibitem[{Quandl(2022)}]{QuandlWIKI}
Quandl. 2022.
\newblock {WIKI} Various End-Of-Day Data.
\newblock \url{https://www.quandl.com/data/WIKI}.
\newblock Accessed: 2022-05-18.

\bibitem[{Rudin(2019)}]{rudin2019stop}
Rudin, C. 2019.
\newblock Stop explaining black box machine learning models for high stakes
  decisions and use interpretable models instead.
\newblock \emph{Nature Machine Intelligence}, 1(5): 206--215.

\bibitem[{Shah et~al.(2022)Shah, Wang, Wilder, Perrault, and
  Tambe}]{shah2022decision}
Shah, S.; Wang, K.; Wilder, B.; Perrault, A.; and Tambe, M. 2022.
\newblock Decision-Focused Learning without Decision-Making: Learning Locally
  Optimized Decision Losses.
\newblock In \emph{Advances in Neural Information Processing Systems}.

\bibitem[{Smith(2017)}]{smith2017cyclical}
Smith, L.~N. 2017.
\newblock Cyclical learning rates for training neural networks.
\newblock In \emph{2017 IEEE winter conference on applications of computer
  vision (WACV)}, 464--472. IEEE.

\bibitem[{Tschiatschek, Sahin, and
  Krause(2018)}]{tschiatschek2018differentiable}
Tschiatschek, S.; Sahin, A.; and Krause, A. 2018.
\newblock Differentiable submodular maximization.
\newblock In \emph{Proceedings of the 27th International Joint Conference on
  Artificial Intelligence}, 2731--2738.

\bibitem[{Wang et~al.(2021)Wang, Shah, Chen, Perrault, Doshi-Velez, and
  Tambe}]{wang2021learning}
Wang, K.; Shah, S.; Chen, H.; Perrault, A.; Doshi-Velez, F.; and Tambe, M.
  2021.
\newblock Learning MDPs from Features: Predict-Then-Optimize for Sequential
  Decision Making by Reinforcement Learning.
\newblock In Ranzato, M.; Beygelzimer, A.; Dauphin, Y.; Liang, P.; and Vaughan,
  J.~W., eds., \emph{Advances in Neural Information Processing Systems},
  volume~34, 8795--8806. Curran Associates, Inc.

\bibitem[{Wang et~al.(2022)Wang, Verma, Mate, Shah, Taneja, Madhiwalla, Hegde,
  and Tambe}]{wang2022decision}
Wang, K.; Verma, S.; Mate, A.; Shah, S.; Taneja, A.; Madhiwalla, N.; Hegde, A.;
  and Tambe, M. 2022.
\newblock Decision-Focused Learning in Restless Multi-Armed Bandits with
  Application to Maternal and Child Care Domain.
\newblock \emph{arXiv preprint arXiv:2202.00916}.

\bibitem[{Wang et~al.(2020)Wang, Wilder, Perrault, and
  Tambe}]{wang2020automatically}
Wang, K.; Wilder, B.; Perrault, A.; and Tambe, M. 2020.
\newblock Automatically learning compact quality-aware surrogates for
  optimization problems.
\newblock \emph{Advances in Neural Information Processing Systems}, 33:
  9586--9596.

\bibitem[{Wilder, Dilkina, and Tambe(2019)}]{wilder2019melding}
Wilder, B.; Dilkina, B.; and Tambe, M. 2019.
\newblock Melding the data-decisions pipeline: Decision-focused learning for
  combinatorial optimization.
\newblock In \emph{Proceedings of the AAAI Conference on Artificial
  Intelligence}, volume~33, 1658--1665.

\bibitem[{Wilder et~al.(2019)Wilder, Ewing, Dilkina, and Tambe}]{wilder2019end}
Wilder, B.; Ewing, E.; Dilkina, B.; and Tambe, M. 2019.
\newblock End to end learning and optimization on graphs.
\newblock \emph{Advances in Neural Information Processing Systems}, 32:
  4672--4683.

\bibitem[{Xie et~al.(2020)Xie, Dai, Chen, Dai, Zhao, Zha, Wei, and
  Pfister}]{xie2020differentiable}
Xie, Y.; Dai, H.; Chen, M.; Dai, B.; Zhao, T.; Zha, H.; Wei, W.; and Pfister,
  T. 2020.
\newblock Differentiable top-k with optimal transport.
\newblock \emph{Advances in Neural Information Processing Systems}, 33:
  20520--20531.

\bibitem[{{Yahoo!}(2007)}]{webscope}
{Yahoo!} 2007.
\newblock {A1 - Yahoo! Search Marketing Advertising Bidding Data, Version 1.0.}
\newblock \url{https://webscope.sandbox.yahoo.com/}.
\newblock Accessed: 2022-05-18.

\end{thebibliography}

\clearpage
\onecolumn
\appendix

\section{Proof of \cref{lemma:weights}}\label{sec:lemmaproof}
\begin{proof}
    The loss value for a prediction $\pred$ made using features $\feature$ for a weight-the-MSE type loss
    $L^{\text{WMSE}}$ is:
    \begin{align*}
        L^{\text{WMSE}}_{\feature}(\preds) &=  \mathbb{E}_{\labels | \feature}[\sum_{n=0}^{N} w^{\labels}_n (\pred_n - \singlelabel_n)^2]
    \end{align*}
    To find the optimal prediction, we differentiate the RHS with respect to the prediction $\pred$ corresponding to the
    feature $\feature$ and equate it to $0$. Importantly, the prediction $\pred$ is \textit{not} dependent on the
    distributions $P(\labels | \feature)$. Moreover, because we assume infinite model capacity, the prediction $\pred$
    can also take any value and is independent of any other prediction $\pred' \in \preds_{-\pred}$. As a result:
    \begin{align*}
        0 &= \frac{\delta}{\delta \pred} \mathbb{E}[\sum_{n=0}^{N} w^{\labels}_n (\pred_n - \singlelabel_n)^2] \\ &= \mathbb{E}[ \frac{\delta}{\delta \pred} \sum_{n=0}^{N} w^{\labels}_n (\pred_n - \singlelabel_n)^2] \\ &= \mathbb{E}[ w^{\labels} (\pred - \singlelabel)] = \mathbb{E}[w^{\labels}] \cdot \pred - \mathbb{E}[w^{\labels} \cdot \singlelabel] \\
        \implies \pred &= \frac{\mathbb{E}_{\labels | \feature}[w^{\labels} \cdot \singlelabel]}{\mathbb{E}_{\labels | \feature}[w^{\labels}]}
    \end{align*}
\end{proof}

\section{Numerical Details for the Counter-Example in \cref{sec:example}}\label{sec:calculation} We give more numerical
details for the motivating example step-wise below:
\begin{itemize}[leftmargin=1.2em]
    \item \textbf{Step 1:} For each instance, $\labels^{\text{\color{figureblue} blue}} =
    (\singlelabel^{\text{\color{figureblue} blue}}_A, \singlelabel^{\text{\color{figureblue} blue}}_B) =
    {\color{figureblue} (0, 0.55)}$ and $\labels^{\text{\color{figureorange} orange}} =
    (\singlelabel^{\text{\color{figureorange} orange}}_A, \singlelabel^{\text{\color{figureorange} orange}}_B) =
    {\color{figureorange} (1, 0.55)}$, we generate 15 possible predictions by adding noise $\epsilon \sim U[-1, 1]$ to
    $\singlelabel^{\text{\color{figureblue} blue}}_A$ and $\singlelabel^{\text{\color{figureorange} orange}}_A$
    respectively. Concretely, we consider the following possible predictions $\samples^{\text{\color{figureblue} blue}}
    \in \{{\color{figureblue} (-1, 0.55)}, {\color{figureblue} (-0.86, 0.55)}, \ldots, {\color{figureblue} (0.86,
    0.55)}, {\color{figureblue} (1, 0.55)}\}$, and $\samples^{\text{\color{figureorange} orange}} \in
    \{{\color{figureorange} (0, 0.55)}, {\color{figureorange} (0.14, 0.55)}, \ldots, {\color{figureorange} (1.86,
    0.55)}, {\color{figureorange} (2, 0.55)}\}\}$.
    \item \textbf{Step 2:} For each of these possible predictions, we calculate the optimal decision $\optfn(\samples)$
    and then the decision quality regret $\dqregret(\samples, \labels) = \obj(\optfn(\labels), \labels) -
    \obj(\optfn(\samples),\labels)$. We document these values in the following table:\vspace{0.5em}\\
    \begin{minipage}{\linewidth}
    \centering
    \begin{tabular}{c c c}
        \toprule
        \makecell{Possible Predicted\\Utilities $\samples$} & \makecell{Optimal Decision\\$\optfn(\samples)$} &
        \makecell{Decision Quality Regret\\$\dqregret(\samples, \labels)$} \\
        \midrule
        {\color{figureblue} (-1, \textbf{0.55})} & Give Resource to \textbf{B} & $\singlelabel^{\text{\color{figureblue}
        blue}}_B - \singlelabel^{\text{\color{figureblue} blue}}_{\mathbf{B}} = 0$ \\
        \llap{9 values $\longrightarrow$ \quad}\vdots & \vdots & \vdots\\
        {\color{figureblue} (0.43, \textbf{0.55})} & Give Resource to \textbf{B} &
        $\singlelabel^{\text{\color{figureblue} blue}}_B - \singlelabel^{\text{\color{figureblue} blue}}_{\mathbf{B}} =
        0$ \\
        {\color{figureblue} (\textbf{0.57}, 0.55)} & Give Resource to \textbf{A} &
        $\singlelabel^{\text{\color{figureblue} blue}}_B - \singlelabel^{\text{\color{figureblue} blue}}_{\mathbf{A}} =
        \mathbf{0.55}$ \\
        \llap{2 values $\longrightarrow$ \quad}\vdots & \vdots & \vdots\\
        {\color{figureblue} (\textbf{1}, 0.55)} & Give Resource to \textbf{A} & $\singlelabel^{\text{\color{figureblue}
        blue}}_B - \singlelabel^{\text{\color{figureblue} blue}}_{\mathbf{A}} = \mathbf{0.55}$ \\
        \midrule
        {\color{figureorange} (0, \textbf{0.55})} & Give Resource to \textbf{B} &
        $\singlelabel^{\text{\color{figureorange} orange}}_A - \singlelabel^{\text{\color{figureorange}
        orange}}_{\mathbf{B}} = \textbf{0.45}$ \\
        \llap{2 values $\longrightarrow$ \quad}\vdots & \vdots & \vdots\\
        {\color{figureorange} (0.43, \textbf{0.55})} & Give Resource to \textbf{B} &
        $\singlelabel^{\text{\color{figureorange} orange}}_A - \singlelabel^{\text{\color{figureorange}
        orange}}_{\mathbf{B}} = \textbf{0.45}$ \\
        {\color{figureorange} (\textbf{0.57}, 0.55)} & Give Resource to \textbf{A} &
        $\singlelabel^{\text{\color{figureorange} orange}}_A - \singlelabel^{\text{\color{figureorange}
        orange}}_{\mathbf{A}} = 0$ \\
        \llap{9 values $\longrightarrow$ \quad}\vdots & \vdots & \vdots\\
        {\color{figureorange} (\textbf{2}, 0.55)} & Give Resource to \textbf{A} &
        $\singlelabel^{\text{\color{figureorange} orange}}_A - \singlelabel^{\text{\color{figureorange}
        orange}}_{\mathbf{A}} = 0$ \\
        \bottomrule
    \end{tabular}
    \end{minipage}
        \item \textbf{Step 3:} Based on the datasets $\{\samples^{\text{\color{figureblue} blue}}\}$ and
        $\{\samples^{\text{\color{figureorange} orange}}\}$ for each instance, we learn weights for the corresponding
        instance by learning a weight for each instance that minimizes the MSE loss. Specifically, we choose:
        \begin{gather*}
            w^{\text{{\color{figuregreen}color}}} = \argmin_{\hat{w}^{\text{{\color{figuregreen}color}}}} \frac{1}{15} \sum_{\samples^{\text{\color{figuregreen}color}}} [\dqregret(\samples^{\text{\color{figuregreen}color}}, \labels^{\text{\color{figuregreen}color}}) - \text{WMSE}_{\hat{w}^{\text{{\color{figuregreen}color}}}}(\samples^{\text{\color{figuregreen}color}}, \labels^{\text{\color{figuregreen}color}})]^2\\
            \text{where,} \quad \text{WMSE}_{\hat{w}^{\text{{\color{figuregreen}color}}}}(\samples^{\text{\color{figuregreen}color}}, \labels^{\text{\color{figuregreen}color}}) \equiv \hat{w}^{\text{{\color{figuregreen}color}}} \cdot (\sample^{\text{\color{figuregreen}color}}_A - \singlelabel^{\text{\color{figuregreen}color}}_A)^2
        \end{gather*}
        To calculate the $\argmin$ we run gradient descent, and get $w^{\text{\color{figureblue} blue}} \approx 0.385$
        and $w^{\text{\color{figureorange} orange}} \approx 0.582$.
    \item \textbf{Step 4:} Based on \cref{lemma:weights}, we know that:
    \begin{align*}
        \pred &= \frac{\mathbb{E}_{\labels | \feature}[w^{\labels} \cdot \singlelabel]}{\mathbb{E}_{\labels | \feature}[w^{\labels}]}
    \end{align*}
    Then, plugging the values of the weights from Step 3 into the formula above, we get:
    \begin{align*}
        \pred_A &= \frac{p^{\text{\color{figureblue} blue}} \cdot w^{\text{\color{figureblue} blue}} \cdot \singlelabel^{\text{\color{figureblue} blue}}_A + p^{\text{\color{figureorange} orange}} \cdot w^{\text{\color{figureorange} orange}} \cdot \singlelabel^{\text{\color{figureorange} orange}}_A}{p^{\text{\color{figureblue} blue}} \cdot w^{\text{\color{figureblue} blue}}+ p^{\text{\color{figureorange} orange}} \cdot w^{\text{\color{figureorange} orange}}}\\ &= \frac{0.5 \cdot 0.385 \cdot 0 + 0.5 \cdot 0.582 \cdot 1}{0.5 \cdot 0.385 + 0.5 \cdot 0.582} \approx 0.602
    \end{align*}
\end{itemize}

\section{Visualizations for the Cubic Top-K domain} \label{sec:localnessapp}
\begin{figure}[H]
    \centering
    \includegraphics[width=0.6\linewidth]{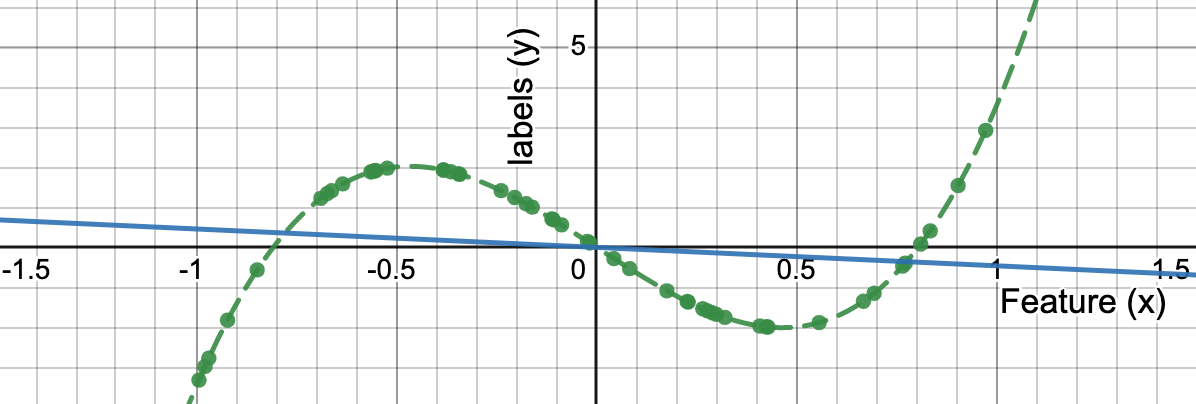}
    \caption{\textbf{Cubic Top-K Domain.} {\color{figuregreen} The underlying mapping from $\feature \to \singlelabel$
    is given by the green dashed line.} The set $\labels$ consists of $N=50$ points where $\feature_n \sim U[-1, 1]$,
    and the goal is to predict the point with the largest $\singlelabel$. {\color{lodlblue} The linear model that
    minimizes the MSE loss is given in blue.}}
    \label{fig:cubic}
\end{figure}

\begin{figure}[H]
\centering
\begin{subfigure}[b]{0.45\linewidth}
     \centering
    \includegraphics[width=\linewidth]{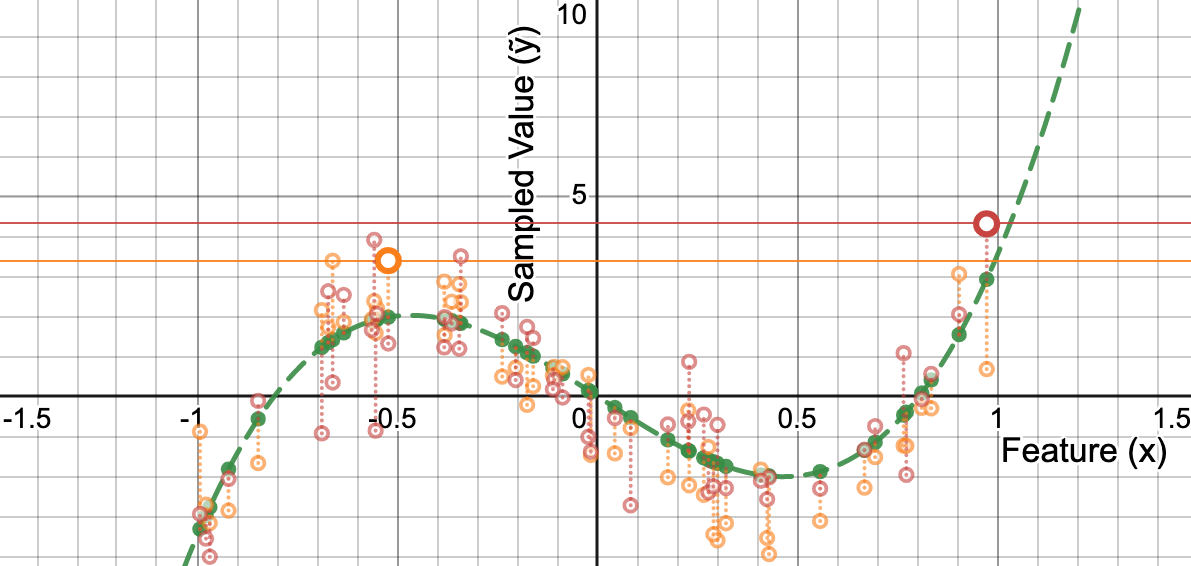}
     \subcaption{Cubic Top-K with Gaussian Sampling}
     \label{fig:cubicgaussian}
\end{subfigure}
\hfil
\begin{subfigure}[b]{0.45\linewidth}
     \centering
    \includegraphics[width=\linewidth]{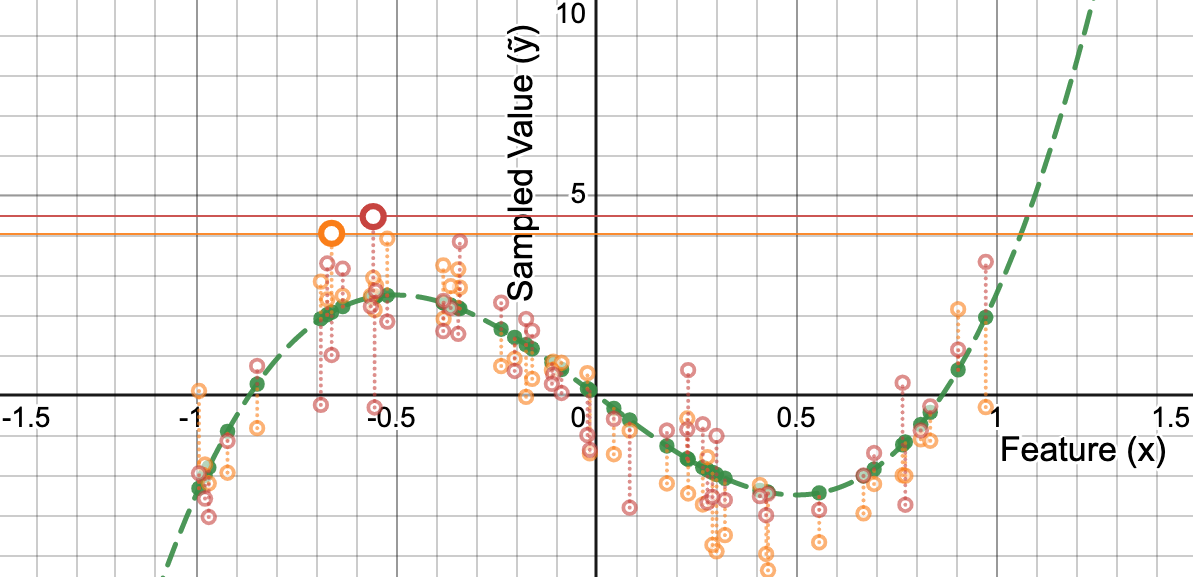}
     \subcaption{(Hard) Cubic Top-K with Gaussian Sampling}
     \label{fig:hardcubicgaussian}
 \end{subfigure}
\\
\begin{subfigure}[b]{0.6\linewidth}
     \centering
    \includegraphics[width=\linewidth]{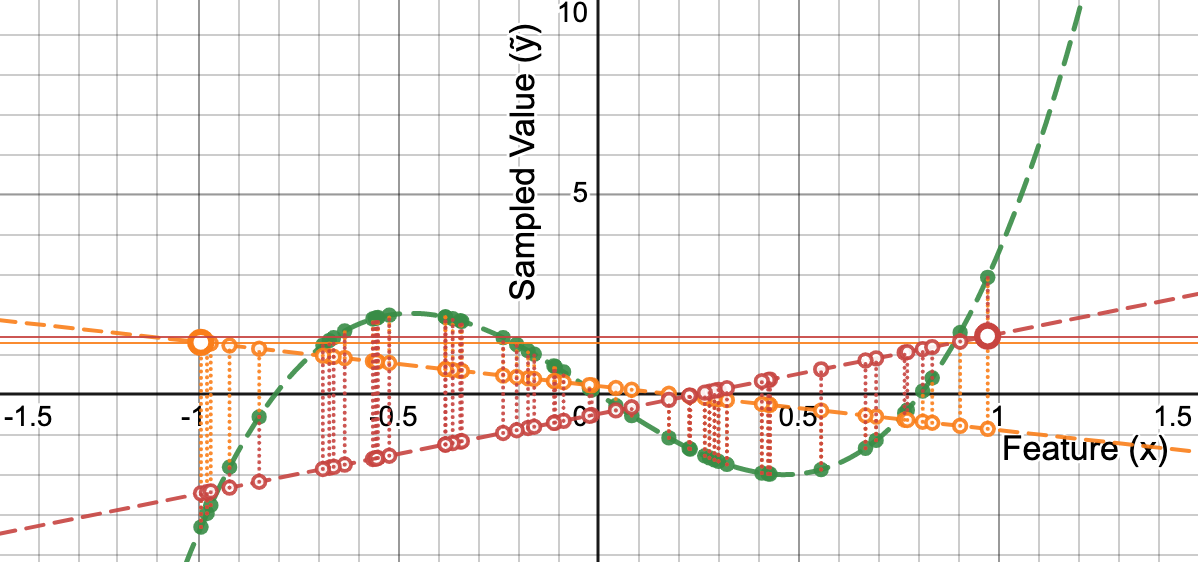}
     \subcaption{Cubic Top-K with Model-based Sampling}
     \label{fig:cubicmbs}
\end{subfigure}
\caption{\textbf{Visualizing Sampling Strategies for the Cubic Top-K Domain.} {\color{figuregreen} The points in green
represent the true labels for some instance $\labels$ with the dashed curve representing the underlying mapping
$\feature \to \singlelabel$.} The points in {\color{figureorange} orange} and {\color{figurered} red} each represent a
set of sampled predictions {\color{figureorange} $\samples_{\text{orange}}$} and {\color{figurered}
$\samples_{\text{red}}$} with the larger point denoting the sampled prediction with the maximum value.}
\label{fig:localness}
\end{figure}

\section{Additional Results}\label{sec:extresults} We run our experiments on an internal
cluster~\cite{OhioSupercomputerCenter1987}. Each job was run with up to 16GM of memory, 8 cores of an Intel Xeon Cascade
Lake CPUs, and optionally one Nvidia A100 GPU. For each method and choice of hyperparameters, we run 10 experiments with
different train-test splits, and randomized initializations of the predictive model and loss function parameters. Then,
for each method, we choose the best hyperparameters based on the highest average decision quality on a validation set.
The corresponding normalized test decision qualities are then reported in Table 3 (for LODLs and EGLs) and Table 1 (for
other models). The best values across different loss function families in Table 3 are then summarized in Table 1. Some
of the hyperparameters that we vary are the learning rates, the number of epochs and patience, the batch size, the loss
function families, and the amount of 2-stage loss we mix into the “expert-crafted surrogates” (in multiples of 10 over
reasonable values). Given the number of hyperparameters, we do not try all combinations but instead manually iterate
over subsets of promising hyperparameter values.

\subsection{Localness of Predictions in Different Domains} \label{sec:localexp}
\begin{table}[H]
    \centering
    \caption{\textbf{Validating the ``localness of predictions'' for different domains}. The error on the validation set
    for predictive models trained on the MSE loss. We see that the localness assumption breaks for the Cubic Top-K
    domains because the errors are high, implying that the predictions are \textit{not} close to the true labels.}
    \label{tab:mseerror}
    \begin{tabular}{cc}
    \toprule
    Domain & Final Validation MSE \\
    \midrule
    Portfolio Optimization & 0.000402 \\
    Web Advertising & 0.063420 \\
    Cubic Top-K & 2.364202 \\
    (Hard) Cubic Top-K & 3.015765 \\
    \bottomrule
    \end{tabular}
\end{table}

\subsection{Analyzing the Portfolio Optimization Domain}\label{sec:portfolioapp}
\begin{table}[H]
\centering
\caption{\textbf{\emph{Validation} DQ for Portfolio Optimization} MBS $\to$ Model-based Sampling, FBP $\to$
Feature-based Parameterization, and EGL = LODL + MBS + FBP. The entries represent the Mean Normalized
\textbf{Validation} DQ $\pm$ SEM. All the EGL variants outperform LODLs on the validation DQ.}
\label{tab:portfolioopt}
\resizebox{0.9\linewidth}{!}{%
\begin{tabular}{cccccc}
\toprule
 \multirow{3}[4]{*}{Domain} & \multirow{3}[4]{*}{Method} & \multicolumn{4}{c}{Normalized \textbf{Validation} Decision
 Quality} \\ 
 \cmidrule(lr){3-6} && \makecell{Directed\\Quadratic} & \makecell{Directed\\WeightedMSE} & Quadratic & WeightedMSE \\
\arrayrulecolor{black}\midrule \multirow{5}[6]{*}{\makecell{Portfolio\\Optimization}} & LODL & 0.170 $\pm$ 0.006 & 0.150
$\pm$ 0.007 & 0.165 $\pm$ 0.006 & 0.134 $\pm$ 0.007 \\ & \textcolor{black!40}{\makecell{LODL (2048 samples)}} &
\textcolor{black!40}{0.189 $\pm$ 0.009} & \textcolor{black!40}{0.162 $\pm$ 0.009} & \textcolor{black!40}{0.165 $\pm$
0.008} & \textcolor{black!40}{0.144 $\pm$ 0.014} \\
\arrayrulecolor{black!40}\cmidrule(lr){2-6} & EGL (MBS) & 0.171 $\pm$ 0.026 & 0.160 $\pm$ 0.027 & 0.179 $\pm$ 0.047 &
0.140 $\pm$ 0.018 \\ & EGL (FBP) & 0.172 $\pm$ 0.009 & 0.175 $\pm$ 0.030 & 0.170 $\pm$ 0.022 & 0.151 $\pm$ 0.009 \\ &
EGL (Both) & \textbf{0.186 $\pm$ 0.020} & \textbf{0.179 $\pm$ 0.024} & \textbf{0.190 $\pm$ 0.022} & \textbf{0.172 $\pm$
0.009} \\
\arrayrulecolor{black}\bottomrule \end{tabular}%
}
\end{table}

\subsection{Time taken to generate dataset for learned losses}\label{sec:datasetgencost}
\begin{figure}[H]
    \centering
    \includegraphics[width=0.4\linewidth]{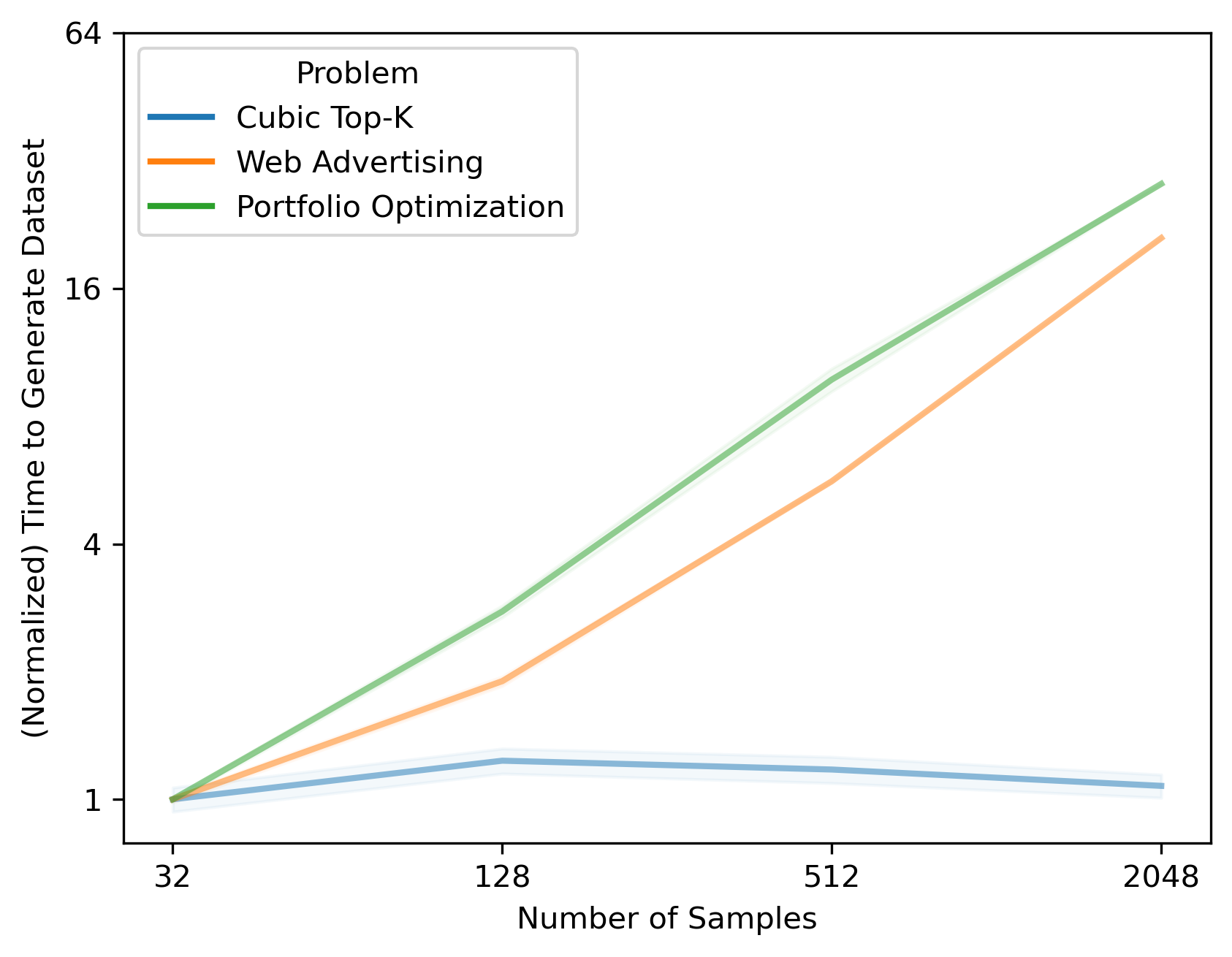}
    \caption{The amount of time taken to create the dataset used to train the loss functions vs. the number of samples per instance $\labels$ in that dataset. To make the results comparable across different experimental domains, we divide the actual generation time by the time taken to generate a dataset containing 32 samples for each domain. We see that for the Web Advertising and Portfolio Optimization domains, the cost scales roughly linearly with the number of samples. For the Cubic Top-K domain, the decision-making problem (top-k) is trivial and the computation is determined by other overheads, leading to a near-constant generation time.}
    \label{fig:timevssamples}
\end{figure}

\subsection{Sensitivity Analysis}
\begin{table}[H]
    \centering
    \caption{\textbf{Varying the complexity of the mapping $P_\psi$ by varying the number of layers in the NN used to
    implement $P$.} We find that, interestingly, whether or not we need a high model complexity depends on the choice of
    the loss function family. For the ‘Quadratic’ loss function families, which perform well in the Web Advertising
    domain (Table 3), we need models with high capacity because of the non-linear mapping between features and
    parameters. Conversely, for ‘WeightedMSE'-type loss functions, which are optimal in the other two domains, lower
    model capacity works better. This is especially true for the Portfolio Optimization domain, where we overfit the
    validation set (\cref{sec:lodlcomparison}).}
    \begin{tabular}{cccc}
\toprule
\multirow{3}[4]{*}{\makecell{Number\\of Layers}} & \multicolumn{3}{c}{Normalized Test DQ}\\
 \cmidrule(lr){2-4} & \makecell{Cubic\\Top-K} & \makecell{Web\\Advertising} & \makecell{Portfolio\\Optimization} \\
\midrule
1 & 0.58 $\pm$ 0.25 & 0.91 $\pm$ 0.04 & \textbf{0.13 $\pm$ 0.02} \\
2 & 0.58 $\pm$ 0.25 & 0.9 $\pm$ 0.02 & 0.12 $\pm$ 0.02 \\
3 & \textbf{0.77 $\pm$ 0.19} & 0.9 $\pm$ 0.01 & 0.12 $\pm$ 0.02 \\
4 & 0.2 $\pm$ 0.21 & \textbf{0.93 $\pm$ 0.01} & 0.12 $\pm$ 0.01 \\
5 & 0.2 $\pm$ 0.31 & 0.92 $\pm$ 0.02 & 0.1 $\pm$ 0.02 \\
\bottomrule
\end{tabular}
\end{table}

\begin{table}[H]
    \centering
    \caption{\textbf{Varying the complexity of the mapping $P_\psi$ by varying the number of layers in the NN used to
    implement $P$.} We find that, in general, a larger number of models is better (8 \& 16 models seem to do better on
    average than 1 \& 2 models).}
    \begin{tabular}{cccc}
\toprule
\multirow{3}[4]{*}{\makecell{Number of\\ Sampling\\Models}} & \multicolumn{3}{c}{Normalized Test DQ}\\
 \cmidrule(lr){2-4} & \makecell{Cubic\\Top-K} & \makecell{Web\\Advertising} & \makecell{Portfolio\\Optimization} \\
\midrule
1 & 0.2 $\pm$ 0.31 & 0.88 $\pm$ 0.04 & 0.11 $\pm$ 0.02 \\
2 & 0.2 $\pm$ 0.31 & 0.89 $\pm$ 0.03 & 0.11 $\pm$ 0.01 \\
4 & 0.4 $\pm$ 0.29 & 0.89 $\pm$ 0.02 & \textbf{0.14 $\pm$ 0.01} \\
8 & 0.2 $\pm$ 0.21 & \textbf{0.93 $\pm$ 0.01} & 0.12 $\pm$ 0.01 \\
16 & \textbf{0.58 $\pm$ 0.25} & 0.88 $\pm$ 0.03 & 0.13 $\pm$ 0.02 \\
\bottomrule
\end{tabular}
\end{table}

\section{Limitations}
\begin{itemize}[leftmargin=1.2em,itemsep=0em,topsep=0.2em]
    \item \textbf{Smaller speed-ups in simple decision-making tasks:} In \cref{sec:timing} we show how a reduction in
    the number of samples needed to train loss functions almost directly corresponds to a speed-up in learning said
    losses. This is because Step 2 of the meta-algorithm (\cref{sec:lodls}), in which we have to run an optimization
    solver for multiple candidate predictions, is the rate-determining step. However, for simpler optimization problems
    that can be solved more quickly (e.g., the Cubic Top-K domain in \cref{fig:timevssamples}), this may no longer be
    the case. In that case, FBP is likely to have limited usefulness. Conversely, however, FBP is likely to be even more
    beneficial for more complex decision-making problems (e.g., MIPs~\cite{ferber2020mipaal} or RL
    tasks~\cite{wang2021learning}).
    \item \textbf{Limited Understanding of Why MBS Performs Well:} In this paper, we endeavor to provide
    \textit{necessary} conditions for when MBS allows EGLs to outperform LODLs, i.e., when the localness assumption is
    broken. However, EGLs seem to work well even when these conditions do not hold, e.g., in the Web Advertising and
    Portfolio Optimization domains. We provide hypotheses for why we believe this occurs in \cref{sec:lodlcomparison},
    but further research is required to rigorously test them.
\end{itemize}


\end{document}